\documentclass{article}

\usepackage[margin=1in]{geometry}
\usepackage{times}
\usepackage{latexsym}
\usepackage{graphicx}
\usepackage{amsmath,amssymb,amsfonts}
\usepackage{booktabs}
\usepackage{hyperref}
\usepackage{xcolor}
\usepackage[numbers]{natbib}

\usepackage[utf8]{inputenc}
\usepackage[T1]{fontenc}
\usepackage{url}
\usepackage{booktabs}
\usepackage{amsfonts}
\usepackage{amsmath,amsthm,amssymb}
\usepackage{nicefrac}
\usepackage{microtype}
\usepackage[dvipsnames]{xcolor}
\usepackage{graphicx}
\usepackage{placeins}
\usepackage{float}
\usepackage{algorithm}
\usepackage{algorithmic}
\usepackage{caption}
\usepackage{enumitem}
\usepackage{wrapfig}
\usepackage{hyperref}

\newtheorem{theorem}{Theorem}
\numberwithin{theorem}{section}
\newtheorem{proposition}[theorem]{Proposition}

\newtheorem{corollary}[theorem]{Corollary}
\newtheorem{assumption}[theorem]{Assumption}
\newtheorem{conjecture}[theorem]{Conjecture}
\newtheorem{definition}[theorem]{Definition}
\newtheorem{remark}[theorem]{Remark}

\title{Flag Varieties:\\
A Geometric Framework for Deep Network Alignment}

\author{
Jingchuan Xiao$^{1}$ \quad Xinyi Sui$^{2}$ \quad Cihan Ruan$^{2}$ \\
$^{1}$Department of Mathematics and Computer Studies, Mary Immaculate College, Ireland \\ $^{2}$Department of Computer Science and Engineering, Santa Clara University, USA \\
}

\date{}

\begin{document}
\vspace{-4mm}
\maketitle
\begin{abstract}
Alignment—the tendency of adjacent weight matrices in deep networks 
to develop compatible subspace orientations—underlies gradient flow, 
Neural Collapse, and representation similarity across architectures. 
Despite extensive empirical documentation, these phenomena have 
resisted unified theoretical treatment: existing explanations are 
post-hoc, each fitted to a specific observation with whatever 
mathematics is at hand. We reverse this direction by deriving the 
mathematical structure that layerwise alignment inherently demands. 
Using geometric invariant theory, we prove that alignment geometry 
has a canonical closed, polystable stratum given by a flag variety, 
and that subspace intersection dimension is its unique 
reparameterization-invariant observable, establishing that subspace 
metrics are not empirical conventions but mathematical necessities. 
This unified framework yields two dynamical consequences: ridge 
regularization drives subspace alignment at an exponential rate set 
by weight decay, whereas nonlinear activations induce a commutator 
obstruction to exact basis alignment, generically present in nonlinear 
networks and absent in linear ones. Together these give a geometric 
explanation of the Level-2/3 hierarchy in Neural Collapse from first 
principles rather than post-hoc analysis. The commutator magnitude 
and head subspace overlap further serve as weight-space windows into 
internal alignment structure, requiring no forward passes. Experiments 
on multilayer perceptrons, residual networks, and pretrained language 
models support the proposed diagnostics and delineate their scope.
\end{abstract}

\vspace{-4mm}
\section{Introduction}

Deep networks exhibit a striking family of alignment phenomena: 
adjacent weight matrices develop compatible subspace 
orientations, terminal-phase features collapse to geometric 
structures, and attention heads in large language models become 
increasingly redundant~\cite{ji2019gradient, saxe2014exact, 
papyan2020prevalence, michel2019heads, voita2019analyzing}. 
These phenomena are usually studied separately—each addressed 
with whatever mathematical tools are convenient—resulting in 
accounts that are post-hoc, fragmented, and unable to generate 
predictions beyond the phenomena they were built to explain. The deeper issue is not that mathematics cannot illuminate deep learning, but that the approach is inverted: existing work fits mathematical tools to observed phenomena, rather than asking what mathematical structure the phenomena themselves demand.

We develop such a language using geometric invariant theory (GIT)~\cite{mumford1994git, borel1991linear}, which provides the natural framework for asking which geometric structures survive arbitrary changes of parameterization.
At a single layer interface, the natural configuration space is
$\mathrm{Gr}(r_1,d)\times\mathrm{Gr}(r_2,d)$ under the action of
$\mathrm{GL}(d)$, corresponding to changes of coordinates in the
interface representation. We prove that, under the natural Plücker linearization, the flag locus is the canonical closed, polystable stratum in the relative-position stratification of this configuration space. Equivalently, it is the maximal-incidence stratum: all lower-incidence strata degenerate toward it under Hilbert--Mumford one-parameter limits~\cite{mumford1994git}. As a consequence, the only full reparameterization-invariant alignment observable is intersection dimension. This has an immediate diagnostic consequence: alignment 
observables that depend on coordinate choices are not intrinsic—only subspace incidence is.

This geometric viewpoint has dynamical consequences and empirical consequences. We formalize a hierarchy of alignment conditions ranging from 
unconstrained configurations to exact basis identification, 
and show that this hierarchy has geometric content: ridge 
regularization drives the classifier row space into the feature 
subspace at an exponential rate controlled by weight decay, giving 
a mechanism for Level-2 subspace alignment. We also identify an 
activation-induced commutator obstruction to exact Level-3 basis 
alignment: nonlinear activations generically introduce anisotropy in the
activation metric, while the obstruction vanishes in the linear/conformal
case. Thus the Level-2/3 hierarchy is not merely an empirical artifact~\cite{papyan2020prevalence, rangamani2023feature},
but follows from the distinction between invariant subspace geometry and
non-invariant basis choices—telling the practitioner not only what
weight decay is geometrically accomplishing during training, but which
alignment targets are reachable and which are not. The same incidence viewpoint gives an exploratory diagnostic for Transformer attention heads: pretrained GPT-2 
small~\cite{radford2019gpt2} and Llama-3.2-1B~\cite{meta2024llama3} exhibit markedly higher head subspace overlap than Haar-random baselines, providing a 
lightweight, training-free probe of head redundancy~\cite{michel2019heads,voita2019analyzing} requiring 
only query projection weights.  

Our contributions are as follows:
\textbf{(1) Moduli-theoretic foundation for subspace metrics} 
(Theorem~\ref{thm:local_flag_moduli}, Corollary~\ref{thm:no_go_alignment}): 
flag configurations form the closed maximal-incidence stratum; 
intersection dimension is the unique $\mathrm{GL}(d)$-invariant 
alignment datum, establishing that subspace metrics such as 
CKA~\cite{kornblith2019similarity} are mathematical necessities, 
not empirical conventions.
\textbf{(2) Neural Collapse explained} 
(Theorems~\ref{thm:l2_attraction}--\ref{thm:obstruction_main}): 
Level-2 alignment is the stable subspace target driven by ridge 
regularization at a rate precisely controlled by weight decay; 
Level-3 is generically obstructed by activation-induced anisotropy.
\textbf{(3) Exploratory extension to Transformers} 
(Proposition~\ref{prop:wd_redundancy}, Section~\ref{sec:transformer_flag}): 
weight decay favors identical head subspaces; pretrained GPT-2 small 
and Llama-3.2-1B show higher subspace overlap than random baselines, 
yielding a principled diagnostic for head 
redundancy~\cite{michel2019heads,voita2019analyzing}.

\vspace{-4mm}
\section{Related Works}
\vspace{-2mm}
Prior work either documents these 
phenomena empirically without theoretical grounding, or offers 
partial theoretical explanations tied to specific settings or 
mathematical tools that do not generalize across phenomena. No 
existing work addresses the invariance-theoretic question that 
underlies all.\vspace{-3mm}
\paragraph{Alignment dynamics and Neural Collapse.}
Understanding how layer representations align during training is
central to explaining gradient flow efficiency, training stability,
and the emergence of structured geometry in learned
representations. Gradient descent aligns consecutive layers in deep linear
networks~\cite{ji2019gradient,saxe2014exact}; related work analyzes
training stability through dynamical isometry and spectral properties
of Jacobians~\cite{pennington2017resurrecting,pennington2018emergence,
tarnowski2019dynamical}. Neural Collapse describes the
terminal-phase geometry in which class means form a simplex equiangular tight frame (ETF) and
classifier weights align with class
means~\cite{papyan2020prevalence,rangamani2023feature}, with subsequent work analyzing
this geometry under various
assumptions~\cite{mixon2022neural,yaras2022neural,
zhu2021geometric}. These works characterize how
alignment emerges in specific settings, but do not address which
aspects of alignment are invariant under reparameterization, nor why
subspace-level alignment is generically more stable than basis-level
alignment.\vspace{-3mm}
\paragraph{Representation similarity and subspace geometry.}
Centered Kernel Alignment
(CKA)~\cite{kornblith2019similarity}, Canonical Correlation Analysis (CCA)
and Singular Vector Canonical Correlation Analysis
(SVCCA)~\cite{raghu2017svcca,morcos2018insights} compare
representations at the subspace level and are empirically more
stable than neuron-wise comparisons, but their preference for
subspace metrics lacks formal justification under reparameterization.
Grassmannian geometry and matrix manifold optimization provide tools
for subspace learning and
comparison~\cite{ham2008grassmann,absil2008optimization}. For
parameter-efficient adaptation, low-rank methods such as
Low-Rank Adaptation (LoRA)~\cite{hu2022lora} and related
approaches~\cite{denil2013predicting} exploit independent per-layer
low-rank factors to reduce the cost of fine-tuning large models.
These works treat subspaces as modeling primitives but do not address
incidence relations between subspaces across layer interfaces, nor
which subspace comparisons are intrinsically meaningful under
coordinate change.
\vspace{-3mm}
\paragraph{Attention heads.}
In Transformer architectures, a practical question is whether all
attention heads contribute meaningfully to model performance. Prior
work identifies redundant or prunable heads through task performance
or importance measures~\cite{michel2019heads,voita2019analyzing}.
Concurrently, Yamagiwa et al.~\cite{yamagiwa2026projection} propose
a principal-angle-based metric for pairwise head-to-head subspace
affinity in mechanistic interpretability. These approaches
characterize redundancy empirically or through functional
relationships between individual head pairs, but do not connect head
redundancy to a broader invariance-theoretic framework for
layer-level alignment.\vspace{-2mm}

\paragraph{Geometric invariant theory.}
GIT~\cite{mumford1994git,borel1991linear} studies the geometry of
quotient spaces under group actions, providing tools to identify
stable configurations and canonical representatives in moduli
problems. It is well developed in algebraic geometry and has been
applied to problems in computer vision and shape
analysis~\cite{ham2008grassmann}, but its application to
deep learning has remained limited.  To our knowledge, flag varieties
have not previously been used as an organizing framework for
alignment in deep networks, and GIT has not been brought to bear on
the question of which alignment configurations are stable under
layerwise reparameterization or which alignment observables are
intrinsically meaningful under coordinate change.

The absence of a common invariance-theoretic foundation across these
lines of work motivates the geometric framework we develop in the
following sections.

\vspace{-3mm}
\section{Flag Varieties: The Geometry of Alignment} \vspace{-2mm}
We introduce the geometric objects underlying our framework: 
alignment levels at layer interfaces, flag varieties as their 
natural parameter space, and the dimensional advantage of 
flag-aligned configurations.\vspace{-2mm}
\subsection{Layer Interfaces and Alignment Geometry}\vspace{-2mm}
We study geometric relations between consecutive layers in deep networks.
At layer $i$, a weight matrix $W_i\in\mathbb{R}^{m_i\times n_i}$ of rank
$r\le\min(m_i,n_i)$ admits a singular value decomposition (SVD) in
Eq.~\eqref{eq:svd}~\cite{kato1995perturbation}.
\begin{equation}
W_i = U_i \Sigma_i V_i^\top ,
\label{eq:svd}
\end{equation}
where $\Sigma_i\in\mathbb{R}^{r\times r}$ is diagonal and
$U_i\in\mathbb{R}^{m_i\times r}$, $V_i\in\mathbb{R}^{n_i\times r}$
have orthonormal columns.

At an interface with compatible dimensions $n_i = m_{i+1}$, the left
singular subspace $\mathrm{col}(U_i)$ of $W_i$ and the right singular subspace $\mathrm{col}(V_{i+1})$ of $W_{i+1}$
both live in the same ambient space, so alignment is a relation between
two points of a Grassmannian.

Rather than individual bases, we focus on the $r$-dimensional subspaces
$\mathrm{col}(U_i)\subset\mathbb{R}^{m_i}$ and
$\mathrm{col}(V_{i+1})\subset\mathbb{R}^{m_i}$
spanned by their columns.
The space of all $r$-dimensional subspaces of $\mathbb{R}^n$
is the Grassmannian $\mathrm{Gr}(r,n)$~\cite{absil2008optimization, ham2008grassmann}.

Alignment phenomena are therefore described by relations between such
subspaces. Inclusion relations correspond to points on the flag variety
$\mathrm{Fl}(r_1,r_2;d)$ (Definition~\ref{def:flag}).

\vspace{-2mm}
\subsection{Flag Varieties and Alignment Levels}
\vspace{-2mm}
Not all alignment phenomena are equally strong. We formalize 
a hierarchy of increasingly strict alignment conditions, 
ranging from unconstrained configurations to exact basis 
identification.
\begin{definition}[Alignment Levels]\label{def:alignment_levels}
For adjacent layers with SVDs $W_i=U_i\Sigma_iV_i^\top$ and
$W_{i+1}=U_{i+1}\Sigma_{i+1}V_{i+1}^\top$, we define a hierarchy of
increasingly stronger alignment conditions:
\begin{itemize}
\item \textbf{Level 0 (General):} no constraint between $\mathrm{col}(U_i)$ and $\mathrm{col}(V_{i+1})$.
\item \textbf{Level 1 (Flag):} $\mathrm{col}(U_i)\subseteq \mathrm{col}(V_{i+1})$ or $\mathrm{col}(V_{i+1})\subseteq \mathrm{col}(U_i)$.
\item \textbf{Level 2 (Subspace):} $\mathrm{col}(U_i)=\mathrm{col}(V_{i+1})$.
\item \textbf{Level 3 (Strict):} $U_i=V_{i+1}$.
\end{itemize}
\end{definition}\vspace{-2mm}
The distinction between these two levels is crucial in nonlinear
networks: as shown in Section~\ref{sec:neural_collapse},
strict alignment is generically obstructed by nonlinear effects,
while subspace alignment remains achievable and stable.

\vspace{-2mm}
\subsection{Flag Varieties}
\vspace{-2mm}
\begin{definition}[Flag Variety]
\label{def:flag}
A \emph{flag variety} $\mathrm{Fl}(d_1,\ldots,d_k;n)$ parametrizes all
nested sequences $\{0\}\subset V_1\subset\cdots\subset V_k\subset\mathbb{R}^n$
with $\dim(V_i)=d_i$.
\end{definition}

Flag varieties naturally encode alignment constraints~\cite{mumford1994git}.
In feedforward networks, alignment arises \emph{locally} at each
layer interface: subspace alignment
$\mathrm{col}(U_i)=\mathrm{col}(V_{i+1})$ corresponds to a two-step flag
in the interface space.
In architectures with sequential composition or additive residual
updates, when the relevant interfaces are expressed in a shared ambient
representation space, flag-compatible alignments concatenate into the
global incidence chain
\[
\mathrm{col}(V_1)\subseteq \mathrm{col}(U_1)
= \mathrm{col}(V_2)\subseteq \mathrm{col}(U_2)
=\cdots,
\]
which is precisely the structure parametrized by a flag variety.

\vspace{-3mm}
\subsection{Geometric Optimality of Flag-Aligned Configurations}
\vspace{-2mm}
We now quantify the geometric advantage of flag-aligned configurations
in terms of redundant degrees of freedom.
Throughout, we consider networks with equal rank $r$ at each layer.

\begin{proposition}[Dimensional Reduction by Flag Alignment]
\label{prop:flag_dimension}
Let $0<r_1\le \cdots \le r_L\le d$. In a common ambient space
$\mathbb{R}^d$, the flag-aligned configuration space\vspace{-1mm}
\[
\mathrm{Fl}(r_1,\ldots,r_L;d)
=
\{V_1\subseteq\cdots\subseteq V_L\subseteq\mathbb{R}^d:
\dim V_i=r_i\}
\]
is a closed subvariety~\cite{mumford1994git} of the independent configuration space
$\prod_{i=1}^L \mathrm{Gr}(r_i,d)$, of strictly smaller dimension
whenever the nesting constraints are non-vacuous.
Thus flag-aligned configurations form a proper lower-dimensional closed
locus inside the space of independently chosen subspaces. \textit{Proof in Appendix~\ref{app:proof_flag_dimension}.}
\end{proposition}

Proposition~\ref{prop:flag_dimension} shows that flag alignment is not a
generic coincidence among independently chosen subspaces, but a
structured incidence condition. Section~\ref{sec:moduli} explains why
this incidence locus is canonical under layerwise reparametrization.
\vspace{-2mm}
\paragraph{Geometric interpretation.}
By Proposition~\ref{prop:flag_dimension}, flag-aligned configurations
occupy a proper closed subvariety of the space of independently chosen
subspaces—a structured lower-dimensional, hence measure-zero, position rather than
a generic coincidence. Nested subspaces are more constrained than independently
chosen ones, and this constraint is what Section~\ref{sec:moduli}
identifies as the canonical polystable geometry under reparametrization.

\begin{corollary}[Minimal Cumulative Drift]
\label{cor:drift}
Among all sequences of subspaces with prescribed dimensions,
the cumulative drift functional
\begin{equation}
\mathcal{D}(V_1,\ldots,V_L)
=\sum_{i=1}^{L-1}\|P_{V_{i+1}}-P_{V_i}\|_F^2
\end{equation}
is minimized precisely by flag-structured configurations. \textit{Proof in Appendix~\ref{app:proof_drift}.}
\end{corollary}

Flag alignment also yields a direct parameter efficiency 
benefit.
\begin{corollary}[Parameter Efficiency]
\label{cor:param}
Sequential systems with flag-structured representations admit, under a shared flag-basis parameterization, $O(dR)$ parameters, where $R = \max_i r_i$,
compared to $O(d\sum_i r_i)$ for independent low-rank updates~\cite{hu2022lora, denil2013predicting} at each layer. \textit{Proof in Appendix~\ref{app:proof_eff}.}
\end{corollary}

\vspace{-3mm}
\section{Invariant Geometry of Alignment}\vspace{-3mm}
\label{sec:moduli}
The preceding results quantify the geometric advantage of flag-aligned
configurations. We now identify the invariant relative-position geometry
underlying layerwise alignment.
\begin{theorem}[Layerwise Necessity of Flag Geometry]
\label{thm:local_flag_moduli}
Fix integers $r_1 \le r_2 \le d$.
Consider the natural action of $\mathrm{GL}(d)$ on
\[
\mathrm{Gr}(r_1,d) \times \mathrm{Gr}(r_2,d).
\]\vspace{-2mm}
Then:
\begin{enumerate}
\item The space admits a stratification by relative position~\cite{mumford1994git} of $(V,U)$,
indexed by $\dim(V\cap U)$.
\item The locus
\begin{equation}
\mathcal{F}_{r_1,r_2}:=\{(V,U):V\subseteq U\}
\label{eq:flag_locus}
\end{equation}
is a closed subvariety~\cite{mumford1994git} canonically isomorphic to the partial flag variety
$\mathrm{Fl}(r_1,r_2;d)$.

\item With respect to the natural Plücker linearization~\cite{mumford1994git},
$\mathcal{F}_{r_1,r_2}$ is the unique polystable relative-position
stratum. Equivalently, it is the closed stratum of maximal intersection
dimension; every lower-incidence stratum has
$\mathcal{F}_{r_1,r_2}$ in its orbit closure and is not polystable.
\end{enumerate}

In particular, flag geometry is the canonical closed, polystable
configuration at the layerwise level. \textit{Proof in Appendix~\ref{app:proof_moduli}.}
\end{theorem}

\vspace{-2mm}
\paragraph{Scope and architectural generality.}
Theorem~\ref{thm:local_flag_moduli} is purely layerwise and independent of
network architecture.
It applies to any model admitting a representation interface $(V,U)$
between successive modules, including feedforward networks, residual
architectures, and attention-based models such as Transformers.

The result is geometric rather than dynamical: it characterizes which
\emph{local} alignment patterns are well-defined at a single interface.
In architectures where representations are updated additively or
composed sequentially—most notably residual networks—such layerwise
flag-compatible interfaces can be coherently extended across
depth to form a \emph{global} flag geometry.
Whether and to what extent training dynamics enforce layerwise
flag-compatible alignment is analyzed in Section~\ref{sec:dynamics}.
Beyond architectural scope, Theorem~\ref{thm:local_flag_moduli} 
has a direct implication for the choice of alignment metric.
\paragraph{Methodological implications: observability and invariance.}
Polystability selects closed representatives of the relative-position
geometry, giving canonical alignment configurations independent of
coordinate choices. The following result makes the corresponding invariance principle precise: the only full $\mathrm{GL}(d)$-invariant
alignment datum is the intersection dimension; practical subspace
metrics such as CKA and principal angles~\cite{kornblith2019similarity, raghu2017svcca, morcos2018insights} are Euclidean relaxations of
this intrinsic incidence geometry.
Thus practical metrics such as principal angles, CCA, CKA, and
Frobenius overlap should be understood as Euclidean relaxations of this
intrinsic incidence geometry.

\begin{corollary}[No-go for Coordinate-Free Alignment Measures]
\label{thm:no_go_alignment}
Let $f \colon \mathrm{Gr}(r_1,d)\times\mathrm{Gr}(r_2,d)\to\mathbb{R}$
be $\mathrm{GL}(d)$-invariant.
Then there exists $\phi\colon\{0,\ldots,\min(r_1,r_2)\}\to\mathbb{R}$
such that $f(V,U)=\phi(\dim(V\cap U))$.
In particular, $f$ cannot distinguish configurations with the same
intersection dimension.
\end{corollary}
\vspace{-4mm}
\begin{proof}
By Theorem~\ref{thm:local_flag_moduli}(1), the $\mathrm{GL}(d)$-orbits
are exactly the nonempty strata $X_k=\{\dim(V\cap U)=k\}$, and the action on each
stratum is transitive.
Hence any $\mathrm{GL}(d)$-invariant $f$ is constant on each $X_k$;
defining $\phi(k)$ to be this value gives $f(V,U)=\phi(\dim(V\cap U))$.
\end{proof}
\vspace{-2mm}
Non-flag configurations admit no canonical representatives: their relative
positions depend on arbitrary coordinate choices and cannot be meaningfully
compared across runs or architectures.
Individual neurons are not stable geometric objects; any invariant alignment
measure must be sought at the subspace level.
This explains why empirical alignment studies—across convolutional neural networks (CNNs), residual networks (ResNets), and Transformers—use subspace-level metrics (principal angles, CCA, CKA) rather than neuron-level comparisons.
We note that principal angles, CCA, CKA, and the Frobenius overlap
$R^{(\ell)}$ used in Section~\ref{sec:transformer_flag} are not
full $\mathrm{GL}(d)$-invariants; they are computable Euclidean
relaxations of the underlying incidence geometry, stable under
orthogonal changes of basis.

\vspace{-3mm}
\section{Subspace Alignment Dynamics}
\label{sec:dynamics}
\vspace{-3mm}
The invariant geometry of Section~\ref{sec:moduli} establishes that
subspace incidence is the intrinsic language for alignment. We now show
how this geometry appears in terminal alignment dynamics:
ridge regularization drives the classifier subspace into the feature
subspace (Level-2 alignment), while nonlinear activations induce a
commutator obstruction to exact basis alignment (Level-3).
Together, these two results give a geometric explanation of the
Level-2/3 alignment hierarchy observed in empirical Neural Collapse
studies~\cite{papyan2020prevalence,rangamani2023feature}.

Both results in this section operate under the following regularity
assumption on the final nonlinear activation. The Level-2 result requires no additional assumptions beyond 
the training setup. The Level-3 obstruction, however, depends 
on the geometry of the activation Jacobian, which we 
characterize via the following regularity condition. Theorem~\ref{thm:l2_attraction} (Level-2 attraction) does not require
this assumption—its proof uses only linear algebra—but
Theorem~\ref{thm:obstruction_main} (Level-3 obstruction) relies on it
to characterize the activation Jacobian.

Throughout this section, $V := \mathrm{Im}(H)$ denotes the penultimate
feature subspace and $P_V$ its orthogonal projector; $U := \mathrm{Im}(W^\top)$
denotes the classifier subspace and $P_U$ its projector, corresponding
to the interface $(V,U)$ of Section~\ref{sec:moduli} at the final layer.
\begin{assumption}[Geometric Consistency]
\label{ass:geometric_consistency}
For the activation map $\phi : \mathbb{R}^{m} \to \mathbb{R}^{m}$
at the last nonlinear interface with Jacobian $J(x)=\nabla\phi(x)$,
assume $\phi$ is approximately conformal on the data manifold
$\mathcal{M}$:
\begin{equation}
\mathbb{E}_{x \sim \mathcal{M}} \bigl[J(x)^{\top} J(x)\bigr]
\approx c\, I_{m}, \qquad c>0 ,
\label{eq:geom_consistency}
\end{equation}
where the expectation is taken almost everywhere on $\mathcal{M}$.
\end{assumption}

\begin{remark}[When does Assumption~\ref{ass:geometric_consistency} hold?]
Approximate conformality arises in rectified linear unit (ReLU) networks
($\mathbb{E}[J^\top J]\approx \tfrac{1}{2}I$),
smooth activations such as Gaussian error linear unit (GELU) or sigmoid linear unit (SiLU) near initialization,
architectures with batch normalization,
and high-dimensional regimes where concentration of measure applies
\cite{pennington2018emergence,hanin2020finite}.
Such conditions have been studied in the dynamical isometry literature
\cite{tarnowski2019dynamical,pennington2017resurrecting}
and do not require isotropic feature covariances.
\end{remark}
\vspace{-5mm}
\subsection{Why Level-2 is the Stable Target}\vspace{-2mm}
\label{sec:level2}
We analyze the terminal training phase, during which features have
approximately converged and classifier dynamics dominate.
This is the regime studied in empirical Neural Collapse
work~\cite{papyan2020prevalence,rangamani2023feature};
the conclusion holds approximately whenever features change slowly
relative to the classifier.

\begin{theorem}[Level-2 Attraction]
\label{thm:l2_attraction}
Consider gradient flow on the ridge-regularized objective
\begin{equation}
\mathcal{L}(W) = \tfrac{1}{2}\|WH - Y\|_F^2
+ \tfrac{\lambda}{2}\|W\|_F^2 ,
\label{eq:ridge_loss}
\end{equation}
where $H\in\mathbb{R}^{d\times n}$ is the fixed feature matrix,
$Y\in\mathbb{R}^{c\times n}$ is the label matrix,
$W\in\mathbb{R}^{c\times d}$ is the trainable classifier, and
$\lambda>0$.

The objective in Eq.~\eqref{eq:ridge_loss} has the unique minimizer
\begin{equation}
W_\ast = YH^\top(HH^\top+\lambda I)^{-1},
\label{eq:ridge_solution}
\end{equation}
which satisfies $W_\ast(I-P_V)=0$, equivalently
$\mathrm{Im}(W_\ast^\top)\subseteq V$.
Moreover, the component orthogonal to $V$ decays exponentially
as shown in Eq.~\eqref{eq:exp_decay}:
\begin{equation}
\|W(t)(I-P_V)\|_F
=e^{-\lambda t}\|W(0)(I-P_V)\|_F.
\label{eq:exp_decay}
\end{equation}
Thus ridge regularization drives the classifier row space into the
feature subspace at rate $\lambda$.
If additionally $\mathrm{rank}(W_\ast)=\mathrm{rank}(H)$,
then $\mathrm{Im}(W_\ast^\top)=V$, which is
\emph{Level-2 feature--classifier alignment}. \textit{Proof in Appendix~\ref{app:proof_l2_attraction}.}
\end{theorem}

The rate $e^{-\lambda t}$ makes a falsifiable prediction:
larger weight decay $\lambda$ should accelerate Level-2 convergence.
The predicted rate dependence on $\lambda$
 is left for future experimental verification. This completes the picture for Level-2 
alignment. Level-3 alignment behaves differently.
\vspace{-3mm}
\subsection{Why Level-3 is Generically Blocked}
\label{sec:neural_collapse}
\vspace{-2mm}
Theorem~\ref{thm:l2_attraction} shows that ridge dynamics drive the
classifier row space into the feature subspace $V$, giving Level-2
alignment. A stronger requirement is Level-3 alignment: a basis-level identification between classifier and feature directions,
beyond equality of their spans. Such a
basis-level statement is sensitive to the activation metric induced by the
final nonlinear activation.

Under Assumption~\ref{ass:geometric_consistency}, we define the
activation-induced metric
\begin{equation}
D^2 := \mathbb{E}_x[D(x)^\top D(x)] .
\label{eq:activation_metric}
\end{equation}
If $D^2 = cI$, the activation is conformal on average and introduces no
preferred directions. In general, however, the anisotropic component of
$D^2$ may fail to preserve the feature subspace $V$. This failure is
measured by the commutator $[D^2,P_V]$.

\begin{theorem}[Commutator obstruction to Level-3 alignment]
\label{thm:obstruction_main}
Let $P_V$ be the orthogonal projector onto the fixed rank-$r$
feature subspace $V\subseteq\mathbb{R}^d$, and let $D^2$ be as above.
Then $[D^2,P_V]=0$ if and only if $V$ is an invariant subspace of $D^2$,
equivalently, the activation-induced metric has no off-diagonal block
between $V$ and $V^\perp$.
For generic $D^2$ with simple spectrum, this invariance condition has
codimension $r(d-r)$ in $\mathrm{Gr}(r,d)$, and is therefore generically
violated. \textit{Proof in Appendix~\ref{app:proof_obstruction}.}
\end{theorem}

Writing $D^2=cI+A$, the conformal part $cI$ commutes with every $P_V$;
the obstruction is entirely due to the anisotropic residual:
$[D^2,P_V]=[A,P_V]$.
Thus near-conformality does not require the obstruction to vanish; it
only makes the obstruction small when $A$ is small or aligned with $V$.

The commutator obstruction suggests a corresponding dynamic picture.
Appendix~\ref{app:dynamic_perspective} gives a heuristic near-conformal
derivation of the residual law \vspace{-2mm}
\begin{equation}
\frac{d}{dt}\|P_U-P_V\|_F^2
\;\approx\;
-\|\nabla \mathcal{R}\|_F^2
+\|[D^2,P_V]\|_F^2.
\label{eq:alignment_dynamics}
\end{equation}
The first term is the usual dissipative alignment term; the second is an
activation-induced residual that vanishes in linear networks and is
generically nonzero in nonlinear networks.
This yields a testable prediction, based on the dynamical law
in Eq.~\eqref{eq:alignment_dynamics}:
linear networks ($D^2=I$) satisfy $[D^2,P_V]=0$ exactly and therefore
do not face this activation-induced obstruction, while nonlinear networks
should exhibit a nonzero commutator associated with basis-level
instability. We verify this directly in Section~\ref{sec:exp_nc}.

\begin{remark}[Refinement of Neural Collapse]
Empirical Neural Collapse studies quantify alignment via
subspace-level metrics such as principal angles
\cite{papyan2020prevalence,rangamani2023feature}.
Our framework explains why this is the right choice:
Level-2 alignment is the robust coordinate-free target, while Level-3
basis identification requires the exceptional condition $[D^2,P_V]=0$.
By Corollary~\ref{thm:no_go_alignment}, intersection dimension is
also the unique reparameterization-invariant alignment measure,
so subspace metrics are not merely a convention but a mathematical
necessity.
Notably, Section~\ref{sec:exp_nc} shows that Level-2 alignment remains
strong even when Assumption~\ref{ass:geometric_consistency} is
violated (e.g.\ Vision Transformer (ViT)-Tiny~\cite{dosovitskiy2021vit} with GELU/ReLU activations), indicating that the empirical Level-2 phenomenon is more robust than
the sufficient conditions used in our dynamic approximation.
\end{remark}

\vspace{-5mm}
\section{Exploratory Extension: Attention Head Subspaces}
\label{sec:transformer_flag}
\vspace{-2mm}
The subspace-incidence viewpoint developed in Sections~\ref{sec:moduli}
and~\ref{sec:dynamics} suggests a natural diagnostic for Transformer
architectures: do attention heads occupy generic or special positions
in the Grassmannian of query subspaces?
We do not claim a complete theory of Transformer training dynamics. Rather, we treat this as an exploratory diagnostic: a 
first test of whether the flag-geometric language extends 
naturally beyond the Neural Collapse setting.
Instead, we show that (i)~a simplified model isolates one mechanism
by which weight decay can favor redundant head configurations,
(ii)~a Frobenius-based subspace overlap metric captures this tendency
in a geometrically principled and dimension-stable way, and
(iii)~pretrained Transformers exhibit markedly higher head subspace
overlap than random baselines.
\vspace{-2mm}
\paragraph{Setup.}
In a Transformer with $h$ attention heads at layer $\ell$,
the query projection matrices $W_Q^{(\ell,i)} \in \mathbb{R}^{d \times d_k}$
define head subspaces
$Q_i^{(\ell)} := \mathrm{col}(W_Q^{(\ell,i)}) \subset \mathbb{R}^d$.
By Corollary~\ref{thm:no_go_alignment}, the unique
$\mathrm{GL}(d)$-invariant summary of this configuration is the
collection of pairwise intersection dimensions.
We use the \emph{Frobenius overlap}
\begin{equation}
R^{(\ell)} :=
\frac{1}{h(h-1)}\sum_{i\neq j}
\frac{1}{d_k}\|P_i^{(\ell)} P_j^{(\ell)}\|_F^2 \in [0,1]
\label{eq:head_overlap}
\end{equation}
as an aggregate measure: $R^{(\ell)}=0$ when all heads are mutually
orthogonal and $R^{(\ell)}=1$ when all heads coincide.
Unlike the dimension-based score $\delta$, this metric is well-defined
even when $h\cdot d_k \ge d$.
\vspace{-6mm}
\subsection{A Simplified Weight-Sharing Model}
\vspace{-2mm}
The following result isolates one mechanism by which Euclidean
regularization can favor redundant head configurations.
It should be read as a toy-model insight, not as a complete model
of Transformer training: the actual coupling between
$W_Q,W_K,W_V,W_O$ and the softmax nonlinearity is not captured here.

\begin{proposition}[Weight Decay Favors Head Sharing]
\label{prop:wd_redundancy}
Consider two attention heads at a single layer with query matrices
$W_Q^{(1)}, W_Q^{(2)} \in \mathbb{R}^{d \times d_k}$.
Fix the sum $W := W_Q^{(1)} + W_Q^{(2)}$.
Then
\begin{equation}\label{eq:wd_sharing}
\|W_Q^{(1)}\|_F^2 + \|W_Q^{(2)}\|_F^2
\;\ge\;
\frac{1}{2}\|W\|_F^2
\end{equation}
with equality if and only if $W_Q^{(1)} = W_Q^{(2)} = W/2$,
i.e., both heads have identical query matrices and hence identical
column spaces, provided $W$ has rank $d_k$. \textit{Proof in Appendix~\ref{app:proof_wd_redundancy}.}
\end{proposition}

Proposition~\ref{prop:wd_redundancy}, via Eq.~\eqref{eq:wd_sharing}, shows that weight decay, under the idealised constraint of fixed aggregate query direction, selects for identical heads.
The degree to which this toy-model mechanism operates in full
multi-head attention—where $W_K$, $W_V$, $W_O$, and the softmax
nonlinearity all interact—remains an open question.
\vspace{-3mm}
\subsection{Empirical Observation}
\vspace{-2mm}
\begin{conjecture}[Static Head-Subspace Incidence]
\label{conj:static_head_incidence}
For attention-based networks trained with weight decay,
terminal query-head subspaces exhibit higher mean pairwise
Frobenius overlap $R^{(\ell)}$ than Haar-random subspaces
of the same dimensions. A stronger dynamical version—predicting that $R^{(\ell)}$
increases monotonically during training—requires training-time
checkpoints and is left for future work.

\end{conjecture}

\vspace{-4mm}
\section{Experiments}
\label{sec:experiments}
\vspace{-3mm}
We report two sets of experiments.
Section~\ref{sec:exp_nc} directly verifies the novel mechanistic
prediction of Theorem~\ref{thm:obstruction_main}: the commutator
$\|[D^2,P_V]\|_F$ should vanish exactly for linear networks and be
nonzero for nonlinear ones, with magnitude associated with the Level-3
alignment gap.
Section~\ref{sec:exp_transformer} provides exploratory evidence for
the subspace-incidence diagnostic across pretrained language models.
Full experimental details and additional architectures appear in
Appendix~\ref{app:experiments}.

\vspace{-2mm}
\subsection{Mechanistic Diagnostic for the Level-3 Obstruction}
\label{sec:exp_nc}
\vspace{-2mm}\paragraph{Setup.}
We evaluate three multilayer perceptron (MLP) architectures trained on CIFAR-10~\cite{krizhevsky2009learning}: a 50-layer 
linear MLP, a 10-layer ReLU MLP (Xavier init), and a 50-layer ReLU 
MLP (orthogonal init). For each model we extract the penultimate 
feature subspace and estimate $D^2$ via a single forward pass; full 
details in Appendix~\ref{app:nc_mechanism}.
\vspace{-4mm}
\begin{table}[h]
\centering
\small
\caption{Commutator diagnostic for Level-3 alignment. Linear networks yield zero commutator, while nonlinear networks exhibit a strictly positive obstruction. Values are mean $\pm$ std.}
\label{tab:commutator}

\begin{tabular}{@{}lccc@{}}
\toprule
Model & $c$ & $\|[D^2,P_V]\|_F$ & Level-3 \\
\midrule
MLP-linear-50        & $1.000\pm0.000$ & $0.000\pm0.000$ & well-def. \\
MLP-ReLU-10 (Xavier) & $0.648\pm0.071$ & $1.219\pm0.334$ & obstructed \\
MLP-ReLU-50 (orth)   & $0.425\pm0.038$ & $1.461\pm0.075$ & obstructed \\
\bottomrule
\end{tabular}
\end{table}

\begin{figure}[h]
\centering
\begin{minipage}[c]{0.55\textwidth}
\centering
\includegraphics[width=0.85\textwidth]{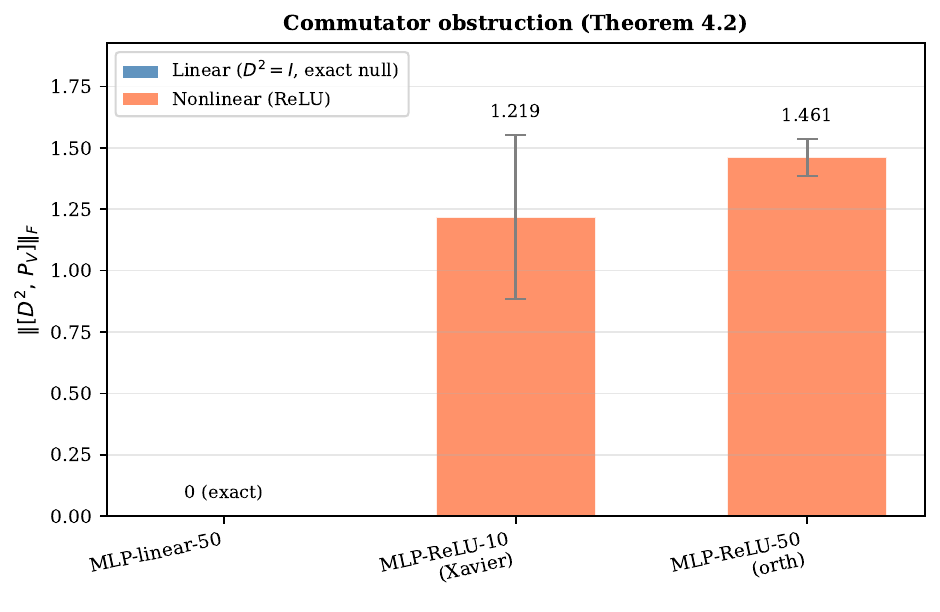}
\end{minipage}
\hfill
\begin{minipage}[c]{0.4\textwidth}
\caption{Commutator $\|[D^2,P_V]\|_F$  across models. Linear network vanishes exactly; nonlinear networks show consistent nonzero obstruction. Error bars: mean $\pm$ std.}
\label{fig:commutator}
\end{minipage}
\end{figure}

\vspace{-3mm}
\paragraph{Results.}




As shown in Table~\ref{tab:commutator} and Figure~\ref{fig:commutator}, 
the commutator behaves exactly as predicted by Theorem~\ref{thm:obstruction_main}. The linear model has $D^2 = I$, so
$[D^2, P_V] = [I, P_V] = 0$ exactly (Table~\ref{tab:commutator}), and Level-3 alignment is
well-defined across all seeds.
Both nonlinear models exhibit a nonzero commutator (Figure~\ref{fig:commutator}),
with the deeper model showing a larger and more stable obstruction
($1.461 \pm 0.075$ vs $1.219 \pm 0.334$).
\vspace{-3mm}
\paragraph{Broader validation.}
Additional architectures (ResNet-18/50~\cite{he2016resnet}, ViT-Tiny~\cite{dosovitskiy2021vit}) in
Appendix~\ref{app:nc_mechanism} support the same pattern: Level-2
alignment remains strong even when Assumption~\ref{ass:geometric_consistency}
is violated (e.g.\ ViT-Tiny with GELU/ReLU), while strict Level-3
diagnostics are gauge-sensitive or collapse to near-zero across all
tested nonlinear models.

\vspace{-4mm}
\subsection{Transformer: Head Subspace Incidence in Pretrained Models}
\label{sec:exp_transformer}
\vspace{-2mm}
\paragraph{Setup.}
We compute the mean pairwise Frobenius overlap $R^{(\ell)}$
(as defined in Eq.~\eqref{eq:head_overlap} in
Section~\ref{sec:transformer_flag}).
In practice, we use the equivalent expression
\begin{equation}
R^{(\ell)} :=
\frac{1}{h(h-1)}\sum_{i\neq j}
\frac{1}{d_k}\bigl\|U_i^{(\ell)\top} U_j^{(\ell)}\bigr\|_F^2,
\label{eq:head_overlap_U}
\end{equation}
which is equivalent to the projector-based definition since
$\|P_i^{(\ell)} P_j^{(\ell)}\|_F^2=
\|U_i^{(\ell)\top} U_j^{(\ell)}\|_F^2.
$
Here $U_i^{(\ell)}$ is the orthonormal column basis of the $i$-th
query projection $W_Q^{(\ell,i)}$, obtained via QR decomposition.
We use $R^{(\ell)}$ rather than a dimension-based score because
$h \cdot d_k \ge d$ in both models, causing dimension-based metrics
to saturate trivially.
We analyze \textbf{GPT-2 small~\cite{radford2019gpt2}} ($d=768$, $h=12$, $d_k=64$, $L=12$)
and \textbf{Llama-3.2-1B~\cite{meta2024llama3}} ($d=2048$, $h=32$, $d_k=64$, $L=16$),
comparing against a Haar-random subspace baseline with analytic
expectation $\mathbb{E}[R^{(\ell)}] = d_k/d$. A randomly initialized GPT-2 small serves as a sanity check.

\vspace{-4mm}
\begin{table}[h]
\centering
\small
\caption{Head-subspace overlap in pretrained Transformers. Pretrained GPT-2 and Llama exceed Haar-random baselines in all layers, while randomly initialized GPT-2 matches the analytic baseline.}
\label{tab:transformer}
\vspace{4pt}
\begin{tabular}{lcccc}
\toprule
Model & Baseline $d_k/d$ & $\bar R$ & Excess & Layers $> 2\sigma$ \\
\midrule
GPT-2 small (random init) & $0.083$ & $0.083$ & $0.000$ & $0/12$ \\
GPT-2 small (pretrained)  & $0.083$ & $0.188$ & $0.105$ & $12/12$ \\
Llama-3.2-1B (pretrained) & $0.031$ & $0.100$ & $0.069$ & $16/16$ \\
\bottomrule
\end{tabular}
\end{table}
\vspace{-4mm}
\begin{figure}[h]
\centering
\includegraphics[width=0.8\textwidth]{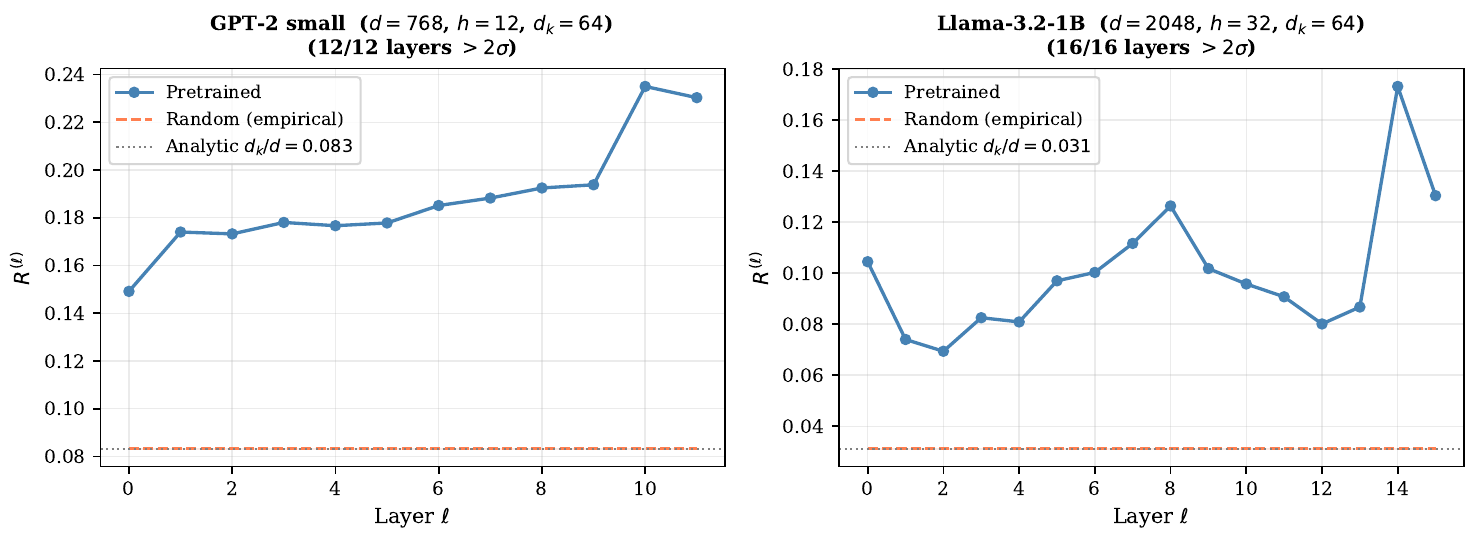}
\caption{Layerwise head-subspace overlap $R^{(\ell)}$. Pretrained GPT-2 and Llama stay above the Haar-random baseline across all layers.}
\label{fig:transformer}
\end{figure}

\vspace{-2mm}
\paragraph{Results.}
\vspace{-2mm}
Both pretrained models exceed the Haar baseline in every layer, 
as shown in Figure~\ref{fig:transformer} and Table~\ref{tab:transformer}, 
while randomly initialized GPT-2 remains at the analytic baseline $d_k/d$. This gap indicates that the excess overlap is a learned property rather than a consequence of dimensionality. GPT-2 also shows an increasing depth profile ($R^{(0)}\approx 0.149$ to $R^{(11)}\approx 0.234$),
consistent with the flag-geometric intuition that compositional
architectures accumulate subspace constraints across layers. The overlap $R^{(\ell)}$ thus serves as a 
geometry-aware diagnostic: it is exactly zero when Level-3 
alignment is theoretically attainable, and robustly nonzero 
when it is not—making the invisible obstruction visible. We interpret this as a static terminal diagnostic consistent with
Conjecture~\ref{conj:static_head_incidence}; checkpoint verification
is left for future work.

\vspace{-3mm}
\section{Discussion and Limitations}\vspace{-3mm}

Our results provide a geometric explanation of known alignment 
phenomena, not a complete training theory: the Level-2 theorem 
models only the terminal phase, the Level-3 obstruction relies on 
near-conformality of the activation Jacobian, and the Transformer 
results are exploratory diagnostics that cannot distinguish a few 
identical heads from many partially overlapping ones. Future work should test checkpoint dynamics, refine head-incidence
metrics, and connect flag structure to generalization or
interpretability. More broadly, the commutator diagnostic and head-subspace overlap metric suggest a complementary perspective on interpretability: rather than
only probing networks through input-output probes, geometric alignment
measures offer a weight-space window into the internal structure that
training has produced.
\vspace{-3mm}
\section{Conclusion}\vspace{-3mm}
Theoretical explanations of alignment phenomena in deep networks 
have largely been post-hoc: mathematics selected after the fact to 
describe what was already observed. We reverse this direction by 
identifying, from first principles, the mathematical structure that 
makes layerwise alignment coordinate-free—and find that flag varieties 
and geometric invariant theory provide a single unified language that 
simultaneously accounts for gradient flow alignment, the Level-2/3 
hierarchy in Neural Collapse, and attention head redundancy in 
Transformers. Theory need not merely trail observation; it can predict and 
constrain it.

The consequences are both theoretical and diagnostic: subspace 
incidence is the unique alignment observable that survives arbitrary 
reparameterization; the commutator magnitude and head subspace overlap 
make previously opaque internal structure measurable without forward 
passes; and weight decay and activation anisotropy become precise, 
falsifiable controls on what alignment training can and cannot 
achieve. The geometry of alignment is not merely describable—it is 
computable.

\vspace{-3mm}
\bibliographystyle{plainnat}
\bibliography{reference}

\clearpage
\section*{Supplementary Information}
\addcontentsline{toc}{section}{Supplementary Information}

\appendix

\setcounter{equation}{0}
\setcounter{figure}{0}
\setcounter{table}{0}
\renewcommand{\theequation}{A.\arabic{equation}}
\renewcommand{\thefigure}{A.\arabic{figure}}
\renewcommand{\thetable}{A.\arabic{table}}

\section{Proofs for Flag Varieties}
\label{app:proof_flag}
\subsection{Proof of Proposition~\ref{prop:flag_dimension} (Dimensional Reduction)}

\label{app:proof_flag_dimension}

\begin{proof}
The inclusion
$\mathrm{Fl}(r_1,\ldots,r_L;d)\subseteq \prod_{i=1}^L \mathrm{Gr}(r_i,d)$
is defined by the incidence conditions $V_i\subseteq V_{i+1}$.
These are closed conditions: the natural maps
$V_i\to \mathbb{R}^d/V_{i+1}$ vanish for all $i$. Hence the flag
configuration space is a closed subvariety.

The dimension of the independent configuration space is
\begin{equation}
\dim \prod_{i=1}^L \mathrm{Gr}(r_i,d)
=\sum_{i=1}^L r_i(d-r_i).
\label{eq:app_independent_dim}
\end{equation}
The standard dimension formula for the partial flag variety is
\begin{equation}
\dim \mathrm{Fl}(r_1,\ldots,r_L;d)
=\sum_{i=1}^L (r_i-r_{i-1})(d-r_i),
\qquad r_0:=0.
\label{eq:app_flag_dim}
\end{equation}
Comparing the two expressions shows that imposing the incidence
conditions removes degrees of freedom; the inequality is strict
whenever there is at least one genuine nesting constraint.
\end{proof}
\subsection{Proof of Corollary~\ref{cor:drift} (Minimal Cumulative Drift)}
\label{app:proof_drift}
\begin{proof}
For subspaces $V,U\subset\mathbb R^d$ of dimensions $r_1,r_2$,
\begin{equation}
\|P_V-P_U\|_F^2
=r_1+r_2-2\operatorname{tr}(P_VP_U)=
r_1+r_2-2\sum_{j=1}^{\min(r_1,r_2)}\cos^2\theta_j ,
\label{eq:app_projection_distance}
\end{equation}
where $\theta_j$ are the principal angles between $V$ and $U$.
This quantity is minimized when all principal angles corresponding to
the smaller subspace vanish, equivalently when the smaller subspace is
contained in the larger one. Thus the minimum is attained exactly by
flag-compatible pairs. Applying this pairwise along the chain gives the
claim.
\end{proof}

\subsection{Proof of Corollary~\ref{cor:param} (Parameter Efficiency)}
\label{app:proof_eff}
\begin{proof}
A basis for $V_L\in\mathrm{Gr}(r_L,d)$ requires $dr_L$ parameters
(up to the $r_L^2$ gauge freedom). Each $V_i$ is then determined by
choosing an $r_i$-dimensional subspace of $V_L$, which requires
$r_i(r_L-r_i)$ additional parameters—all within the already-stored
basis. The total is $O(dr_L)$. By contrast, independent bases for
$V_1,\ldots,V_L$ require $\sum_i dr_i = O(d\sum_i r_i)$ parameters.
\end{proof}

\setcounter{equation}{0}
\renewcommand{\theequation}{B.\arabic{equation}}

\section{Proof of Theorem~\ref{thm:local_flag_moduli} (Moduli-Theoretic Necessity)}
\label{app:proof_moduli}
\begin{theorem}[Theorem~\ref{thm:local_flag_moduli}, restated]
Fix integers $r_1 \le r_2 \le d$.
Consider the natural action of $\mathrm{GL}(d)$ on
\[
\mathrm{Gr}(r_1,d) \times \mathrm{Gr}(r_2,d).
\]

Then:

\begin{enumerate}
\item The space admits a stratification by relative position of $(V,U)$,
indexed by $\dim(V \cap U)$.

\item The locus
\begin{equation}
\mathcal{F}_{r_1,r_2} := \{(V,U) : V\subseteq U\}
\label{eq:appB_flag_locus}
\end{equation}
is a closed subvariety canonically isomorphic to the partial flag variety
$\mathrm{Fl}(r_1,r_2;d)$.

\item The locus $\mathcal F_{r_1,r_2}$ is the unique polystable stratum.
Every lower-incidence stratum admits a one-parameter degeneration whose
limit lies in $\mathcal{F}_{r_1,r_2}$, and hence is not polystable.
\end{enumerate}

In particular, flag geometry is the unique geometrically stable
configuration at the layerwise level.
\end{theorem}
\begin{proof}
We consider the natural action of $G := \mathrm{GL}(d)$ on
\begin{equation}
X := \mathrm{Gr}(r_1,d)\times \mathrm{Gr}(r_2,d),
\label{eq:appB_X_def}
\end{equation}
given by $g\cdot(V,U) := (gV, gU)$.

\paragraph{1. Stratification by relative position.}
The $G$-orbits in $X$ are classified by the relative position of the pair
$(V,U)$, equivalently by the integer
\begin{equation}
k := \dim(V\cap U),
\qquad 0 \le k \le r_1.
\label{eq:appB_k_def}
\end{equation}
This yields a finite stratification of $X$ into locally closed subsets
\begin{equation}
X = \bigsqcup_{k=0}^{r_1} X_k,
\quad
X_k := \{(V,U) : \dim(V\cap U)=k\}.
\label{eq:appB_stratification}
\end{equation}
This is standard: $\mathrm{GL}(d)$ acts transitively on pairs of subspaces with fixed dimensions and fixed intersection dimension, so each nonempty $X_k$ is a single $G$-orbit.

\paragraph{2. The flag locus as a closed subvariety.}
The locus
\begin{equation}
\mathcal{F}_{r_1,r_2} := \{(V,U)\in X : V\subseteq U\}
\label{eq:appB_flag_locus2}
\end{equation}
coincides with the stratum $X_{r_1}$.
It is Zariski closed: the condition $V\subseteq U$ is equivalent to
the vanishing of the natural map $V\to\mathbb{R}^d/U$, hence is given
by rank equations on the product of Grassmannians.
Moreover, the map
\[
(V,U)\longmapsto (V\subset U\subset \mathbb{R}^d)
\]
identifies $\mathcal{F}_{r_1,r_2}$ canonically with the partial flag variety
$\mathrm{Fl}(r_1,r_2;d)$.

\paragraph{3. Uniqueness of the flag stratum.}

\textit{Step 1 (Non-polystability of non-flag strata).}
Let $(V,U)\in X_k$ with $k<r_1$, where
\begin{equation}
X_k=\{(V,U):\dim(V\cap U)=k\}.
\label{eq:appB_Xk_def}
\end{equation}
Set $s:=r_1-k$. Since $\dim(V+U)=r_1+r_2-k=r_2+s\le d$,
after a change of basis we may write
\begin{equation}
U=\mathrm{span}(e_1,\dots,e_{r_2}),
\label{eq:appB_U_form}
\end{equation}
and
\begin{equation}
V=\mathrm{span}\bigl(
e_1,\dots,e_k,\,
e_{r_2+1}+e_{k+1},\dots,
e_{r_2+s}+e_{k+s}
\bigr).
\label{eq:appB_V_form}
\end{equation}
Then $\dim(V\cap U)=k$.

We use a standard one-parameter subgroup degeneration argument, as in
the Hilbert--Mumford criterion of geometric invariant theory
\cite{mumford1994git}. Define a one-parameter subgroup
$\lambda:\mathbb{G}_m\to \mathrm{GL}(d)$ by
\begin{equation}
\lambda(t)e_i =
\begin{cases}
t\,e_i, & i=r_2+1,\dots,r_2+s,\\
e_i, & \text{otherwise},
\end{cases}
\label{eq:appB_lambda}
\end{equation}
where $\mathbb{G}_m$ denotes the multiplicative group $\mathbb{R}^{\times}$.
Then $\lambda(t)\cdot U=U$ for all $t$, while for $j=1,\dots,s$,
\begin{equation}
\lambda(t)\bigl(e_{r_2+j}+e_{k+j}\bigr)
=t e_{r_2+j}+e_{k+j}
\longrightarrow e_{k+j}
\qquad (t\to0).
\label{eq:appB_limit_vector}
\end{equation}
Hence, taking the limit in $\mathrm{Gr}(r_1,d)$ via the Plücker embedding,
\begin{equation}
\lim_{t\to0}\lambda(t)\cdot V
=V':=
\mathrm{span}(e_1,\dots,e_k,e_{k+1},\dots,e_{k+s})
\subset U.
\label{eq:appB_limit_V}
\end{equation}
Therefore
\begin{equation}
\lim_{t\to0}\lambda(t)\cdot(V,U)=(V',U)\in\mathcal{F}_{r_1,r_2}.
\label{eq:appB_limit_pair}
\end{equation}

This shows that the $G$-orbit of $(V,U)$ is not closed and that its closure
intersects the flag locus $\mathcal{F}_{r_1,r_2}$.
Consequently, every stratum $X_k$ with $k<r_1$ consists of non-polystable
points in the Plücker-linearized relative-position quotient.

\textit{Step 2 (Polystability and uniqueness).}
For any $(V\subset U)\in\mathcal{F}_{r_1,r_2}$, let
\begin{equation}
P:=\mathrm{Stab}_G(V,U).
\label{eq:appB_stabilizer}
\end{equation}
The subgroup $P$ contains the stabilizer of a complete flag refining
$V\subset U$, hence contains a Borel subgroup. By the standard criterion
that a closed subgroup of a linear algebraic group is parabolic if and
only if it contains a Borel subgroup \cite{borel1991linear}, $P$ is
parabolic.

It follows that
\begin{equation}
G\cdot(V,U)\simeq \mathrm{GL}(d)/P
\label{eq:appB_orbit}
\end{equation}
is projective \cite{borel1991linear}. Since the orbit map
\[
\mathrm{GL}(d)/P
\longrightarrow
\mathrm{Gr}(r_1,d)\times\mathrm{Gr}(r_2,d)
\]
has image $G\cdot(V,U)=\mathcal{F}_{r_1,r_2}$, this image is closed.
Thus all points in $\mathcal{F}_{r_1,r_2}$ have closed orbits.

By the Hilbert--Mumford criterion, closed orbits in the semistable locus
are polystable \cite{mumford1994git}. Hence all points in
$\mathcal{F}_{r_1,r_2}$ are polystable.

On the other hand, Step~1 shows that every non-flag orbit has
$\mathcal{F}_{r_1,r_2}$ in its closure and is not closed, hence is not
polystable. Therefore $\mathcal{F}_{r_1,r_2}$ is the unique polystable
stratum of the layerwise configuration space.
\paragraph{Conclusion.}
The layerwise configuration space admits a finite stratification by relative position,
with a unique geometrically stable (polystable) locus given by the flag variety
$\mathrm{Fl}(r_1,r_2;d)$.
This completes the proof.
\end{proof}

\setcounter{equation}{0}
\renewcommand{\theequation}{C.\arabic{equation}}
\section{Proofs for Neural Collapse}
\label{app:proof_nc}
\subsection{Proof of Theorem~\ref{thm:l2_attraction} (Level-2 Attraction)}
\label{app:proof_l2_attraction}
\begin{proof}
The gradient flow satisfies
\begin{equation}
\dot{W} = -(WH - Y)H^\top - \lambda W.
\label{eq:appC_gradient_flow}
\end{equation}

\medskip
\noindent

Since $H^\top(I - P_V) = 0$, multiplying on the right gives
\begin{equation}
\frac{d}{dt}[W(I - P_V)] = -\lambda W(I - P_V),
\label{eq:appC_proj_dynamics}
\end{equation}
hence
\begin{equation}
W(t)(I - P_V) = e^{-\lambda t} W(0)(I - P_V).
\label{eq:appC_proj_solution}
\end{equation}

\medskip
\noindent

The loss is strictly convex, so $W(t) \to W_\ast$, where
\begin{equation}
W_\ast = YH^\top(HH^\top + \lambda I)^{-1}.
\label{eq:appC_W_star}
\end{equation}
Since $\mathrm{Im}(HH^\top) \subseteq V$ and $HH^\top$ acts trivially on $V^\perp$,
we have $(HH^\top + \lambda I)|_{V^\perp} = \lambda I$.
Therefore
\begin{equation}
W_\ast(I - P_V) = YH^\top(HH^\top + \lambda I)^{-1}(I - P_V)
= \tfrac{1}{\lambda} YH^\top(I - P_V) = 0.
\label{eq:appC_Wstar_proj_zero}
\end{equation}
\medskip
\noindent

Thus $\mathrm{Im}(W_\ast^\top) \subseteq V$.
If $\mathrm{rank}(W_\ast)=\mathrm{rank}(H)$, then equality holds:
\begin{equation}
\mathrm{Im}(W_\ast^\top)=V.
\label{eq:appC_image_equality}
\end{equation}
Since $W(t) \to W_\ast$, the corresponding subspaces converge.
\end{proof}

\subsection{Proof of Theorem~\ref{thm:obstruction_main}
(Commutator Obstruction)}
\label{app:proof_obstruction}

\begin{theorem}[Commutator obstruction, restated]
Let $P_V$ be the orthogonal projector onto an $r$-dimensional subspace
$V\subseteq\mathbb{R}^d$, and let $M:=D^2=\mathbb{E}_x[D(x)^\top D(x)]$
be symmetric positive semidefinite. Then $[M,P_V]=0$ if and only if $V$
is an invariant subspace of $M$. If $M$ has simple spectrum (i.e., distinct eigenvalues), this
condition has codimension $r(d-r)$ in $\mathrm{Gr}(r,d)$.
\end{theorem}

\begin{proof}
By definition, $[M,P_V]=0$ if and only if $MP_V=P_VM$.

\textbf{($\Rightarrow$)} If $MP_V=P_VM$, then for any $v\in V$,
$Mv = MP_Vv = P_VMv \in V$, so $V$ is $M$-invariant.

\textbf{($\Leftarrow$)} If $V$ is $M$-invariant, then for any
$w\in V^\perp$ and $v\in V$, since $M$ is symmetric,
\begin{equation}
\langle Mw,v\rangle = \langle w,Mv\rangle = 0
\label{eq:appC_symmetry_inner}
\end{equation}
because $Mv\in V$ and $w\perp V$. Hence $V^\perp$ is also
$M$-invariant. With respect to the orthogonal decomposition
$\mathbb{R}^d = V\oplus V^\perp$, we may write
\begin{equation}
M = \begin{pmatrix} M_{VV} & 0 \\ 0 & M_{V^\perp V^\perp} \end{pmatrix},
\qquad
P_V = \begin{pmatrix} I & 0 \\ 0 & 0 \end{pmatrix},
\label{eq:appC_block_forms}
\end{equation}
so $[M,P_V]=0$.

In general, $[M,P_V]$ equals
\[
\begin{pmatrix} 0 & -M_{VV^\perp} \\ M_{V^\perp V} & 0 \end{pmatrix},
\]
which vanishes if and only if the off-diagonal blocks $M_{VV^\perp}$
and $M_{V^\perp V}$ are zero, i.e., $V$ is $M$-invariant.

Finally, if $M$ has simple spectrum, every $M$-invariant $r$-dimensional
subspace is spanned by a choice of $r$ eigenvectors, forming a finite
set in $\mathrm{Gr}(r,d)$. Since $\dim\mathrm{Gr}(r,d)=r(d-r)$, the
invariance condition has codimension $r(d-r)$.
\end{proof}

\begin{remark}[Relation to activation-transformed subspaces]
For a fixed symmetric positive definite Jacobian $D$ (the non-averaged case), $[D^2,P_V]=0$ is equivalent
to $DV=V$: the commutator measures whether the activation map preserves
the feature subspace.
\end{remark}

\subsection{Dynamic Perspective}
\label{app:dynamic_perspective}

Theorem~\ref{thm:obstruction_main} identifies the algebraic obstruction.
We now give a heuristic near-conformal derivation showing how the same
commutator appears as a residual term in alignment dynamics.

\paragraph{Setup.}
Write $D^2=cI+\Delta$ where $\Delta$ is symmetric (since $D^2$ is
symmetric) and controlled under
Assumption~\ref{ass:geometric_consistency}. The commutator satisfies
$[D^2,P_V]=[\Delta,P_V]$. In an orthonormal basis adapted to $V\oplus V^\perp$, since $\Delta$
is symmetric,
\begin{equation}
\|[D^2,P_V]\|_F^2
=
2\|P_V^\perp \Delta P_V\|_F^2,
\label{eq:appC_commutator_norm}
\end{equation}
which is the squared Frobenius norm of the off-diagonal block of $\Delta$
between $V$ and $V^\perp$.

\paragraph{Near-equilibrium alignment dynamics.}
Approximating the classifier subspace dynamics by gradient descent on
\begin{equation}
\mathcal{R}(P_U,P_V):=\|P_U-P_V\|_F^2,
\label{eq:appC_alignment_energy}
\end{equation}
and the feature subspace drift by the leading-order activation-induced
term
\begin{equation}
\dot P_V \approx \mathcal{L}_{P_V}(\Delta),
\qquad
\mathcal{L}_{P_V}(\Delta)
:=
P_V^\perp\Delta P_V+P_V\Delta P_V^\perp,
\label{eq:appC_L_operator}
\end{equation}
we compute
\begin{equation}
\frac{d}{dt}\mathcal{R}(P_U,P_V)
=
2\langle P_U-P_V,\dot P_U-\dot P_V\rangle_F.
\label{eq:appC_R_derivative}
\end{equation}
As a reduced near-equilibrium model, we approximate the classifier
subspace dynamics by gradient descent on the alignment energy:
\begin{equation}
\dot P_U\approx -\nabla_{P_U}\mathcal R.
\label{eq:appC_PU_dynamics}
\end{equation}
Since
\begin{equation}
\mathcal R(P_U,P_V)=\|P_U-P_V\|_F^2,
\qquad
\nabla_{P_U}\mathcal R=2(P_U-P_V),
\label{eq:appC_R_gradient}
\end{equation}
we have
\begin{equation}
2\langle P_U-P_V,-\nabla_{P_U}\mathcal R\rangle_F
=
-\|\nabla_{P_U}\mathcal R\|_F^2.
\label{eq:appC_dissipation}
\end{equation}
Together with the leading-order feature-subspace drift
\begin{equation}
\dot P_V\approx \mathcal L_{P_V}(\Delta)+O(\|\Delta\|^2),
\label{eq:appC_PV_dynamics_approx}
\end{equation}
this gives
\begin{equation}
\frac{d}{dt}\|P_U-P_V\|_F^2
\approx
-\|\nabla_{P_U}\mathcal{R}\|_F^2
-
2\langle P_U-P_V,\mathcal{L}_{P_V}(\Delta)\rangle_F
+
O(\|\Delta\|^2).
\label{eq:appC_alignment_dynamics}
\end{equation}

The second term is the forcing term induced by the nonlinear activation.
By Cauchy--Schwarz,
\begin{equation}
\bigl|
2\langle P_U-P_V,\mathcal{L}_{P_V}(\Delta)\rangle_F
\bigr|
\le
2\|P_U-P_V\|_F\,\|\mathcal{L}_{P_V}(\Delta)\|_F.
\label{eq:appC_CS_bound}
\end{equation}
Moreover, in any orthonormal basis adapted to
$\mathbb R^d=V\oplus V^\perp$,
\begin{equation}
\|\mathcal{L}_{P_V}(\Delta)\|_F
=
\|[D^2,P_V]\|_F
\label{eq:appC_L_commutator}
\end{equation}
because both quantities are determined by the off-diagonal block
$P_V^\perp\Delta P_V$.

Near equilibrium, the dissipative alignment term balances the
activation-induced forcing. Since
$\|\mathcal{L}_{P_V}(\Delta)\|_F=\|[D^2,P_V]\|_F$, we use
\begin{equation}
\mathcal K_{\mathrm{NC}}:=\|[D^2,P_V]\|_F^2
\label{eq:appC_K_NC}
\end{equation}
as a scalar diagnostic for the residual drift, leading to the heuristic law
\begin{equation}
\frac{d}{dt}\|P_U-P_V\|_F^2
\approx
-\|\nabla_{P_U}\mathcal{R}\|_F^2+\mathcal K_{\mathrm{NC}}.
\label{eq:appC_final_dynamics}
\end{equation}
Here $\mathcal K_{\mathrm{NC}}$ is precisely the squared off-diagonal mixing of the activation metric between $V$ and $V^\perp$.

\setcounter{equation}{0}
\renewcommand{\theequation}{D.\arabic{equation}}
\section{Proof of Proposition \ref{prop:wd_redundancy} (Weight Decay Favors Head Sharing)}
\label{app:proof_wd_redundancy}

We prove that among all decompositions $W_Q^{(1)} + W_Q^{(2)} = W$
with $W$ fixed, the weight-decay penalty
$\|W_Q^{(1)}\|_F^2 + \|W_Q^{(2)}\|_F^2$ is minimized uniquely at
$W_Q^{(1)} = W_Q^{(2)} = W/2$.

\begin{proof}

For any two matrices $A, B \in \mathbb{R}^{d \times d_k}$, the
parallelogram identity in the Frobenius inner product space states:
\begin{equation}
\|A\|_F^2 + \|B\|_F^2
= \frac{1}{2}\|A + B\|_F^2 + \frac{1}{2}\|A - B\|_F^2.
\label{eq:appD_parallelogram}
\end{equation}
This follows entry-wise from the real identity
$a^2 + b^2 = \frac{1}{2}(a+b)^2 + \frac{1}{2}(a-b)^2$
and linearity of the Frobenius norm squared.

Setting $A = W_Q^{(1)}$ and $B = W_Q^{(2)}$, and using
$W_Q^{(1)} + W_Q^{(2)} = W$:
\begin{equation}
\|W_Q^{(1)}\|_F^2 + \|W_Q^{(2)}\|_F^2
= \frac{1}{2}\|W\|_F^2 + \frac{1}{2}\|W_Q^{(1)} - W_Q^{(2)}\|_F^2.
\label{eq:appD_W_decomposition}
\end{equation}

Since $\|W_Q^{(1)} - W_Q^{(2)}\|_F^2 \ge 0$, we immediately obtain
\begin{equation}
\|W_Q^{(1)}\|_F^2 + \|W_Q^{(2)}\|_F^2 \ge \frac{1}{2}\|W\|_F^2
\label{eq:appD_lower_bound}
\end{equation}
with equality if and only if $\|W_Q^{(1)} - W_Q^{(2)}\|_F^2 = 0$,
i.e., $W_Q^{(1)} = W_Q^{(2)}$.

Under the equality condition $W_Q^{(1)} = W_Q^{(2)}$ and the
constraint $W_Q^{(1)} + W_Q^{(2)} = W$, we get
$2W_Q^{(1)} = W$, hence $W_Q^{(1)} = W_Q^{(2)} = W/2$.

When $W$ has rank $d_k$, both matrices $W/2$ and $W$ have the same
column space $\mathrm{col}(W)$. Therefore, at the minimizer,
both heads share the identical query column space
$\mathrm{col}(W_Q^{(1)}) = \mathrm{col}(W_Q^{(2)}) = \mathrm{col}(W)$.

\textbf{Interpretation.}
The penalty $\|W_Q^{(1)}\|_F^2 + \|W_Q^{(2)}\|_F^2$ decomposes as
\begin{equation}
\underbrace{\frac{1}{2}\|W\|_F^2}_{\text{fixed by constraint}}
+
\underbrace{\frac{1}{2}\|W_Q^{(1)} - W_Q^{(2)}\|_F^2}_{\text{head disagreement}}.
\label{eq:appD_interpretation}
\end{equation}
Weight decay minimizes the total penalty; since the first term is
fixed, it exclusively penalizes head disagreement. The unique minimum
is achieved precisely when both heads agree: $W_Q^{(1)} = W_Q^{(2)}$.
This shows that weight decay, under the idealized constraint of fixed
aggregate query $W$, selects for maximally redundant (identical) heads.
\end{proof}

\paragraph{Remark (toy model scope).}
The proposition fixes the aggregate query matrix $W = W_Q^{(1)} + W_Q^{(2)}$,
which is a strong idealization. In practice, both the aggregate and the
individual heads are jointly optimized, and $W_K$, $W_V$, $W_O$, and
the softmax nonlinearity all interact. The proposition isolates the
weight-decay mechanism in its purest form; whether and to what degree
this mechanism dominates in full multi-head attention is an empirical
question addressed by the diagnostic in Section~\ref{sec:exp_transformer}.

\setcounter{equation}{0}
\setcounter{figure}{0}
\setcounter{table}{0}
\renewcommand{\theequation}{E.\arabic{equation}}
\renewcommand{\thefigure}{E.\arabic{figure}}
\renewcommand{\thetable}{E.\arabic{table}}

\section{Experimental Details}
\label{app:experiments}

\subsection{Part A: Neural Collapse Mechanistic Validation}
\label{app:nc_mechanism}

\subsubsection{Models and Training}

All MLP models are trained on CIFAR-10 with standard cross-entropy loss,
stochastic gradient descent (SGD) with momentum 0.9, weight decay $10^{-4}$, and cosine learning
rate schedule. Inputs are flattened to $32 \times 32 \times 3 = 3072$
dimensions; the final linear classifier maps to 10 classes.

\begin{description}[leftmargin=1.2em,itemsep=0.3em]
\item[\textbf{MLP-linear-50.}]
  50 fully connected layers, no nonlinearity (linear activations),
  Xavier uniform initialization~\cite{glorot2010understanding}.
  5 seeds evaluated.
\item[\textbf{MLP-ReLU-10.}]
  10 fully connected layers with ReLU activations,
  Xavier uniform initialization.
  5 seeds evaluated.
\item[\textbf{MLP-ReLU-50 (orth).}]
  50 fully connected layers with ReLU activations,
  orthogonal initialization (referred to as ``flag-style'' initialization in our implementation).
  We use orthogonal rather than Xavier initialization because
  Xavier on 50-layer ReLU leads to dying neurons
  ($D^2_\text{mean} \approx 0.02$, commutator $\approx 0$),
  which conflates the dying-neuron phenomenon with the commutator obstruction.
  15 seeds evaluated.
\end{description}

\subsubsection{Feature and D² Extraction}

For each model, we extract the penultimate feature representation
$H \in \mathbb{R}^{n \times d}$ (activations before the final linear layer)
and compute the feature subspace basis
$U_\mathrm{feat} \in \mathbb{R}^{d \times r}$
via truncated SVD of $H$ (retaining $r$ components explaining $> 99\%$
of variance).

We estimate $D^2_\mathrm{diag} = \mathbb{E}[\phi'(z)^2]$ by registering
a forward hook on the activation module immediately preceding the final
linear layer.
For ReLU: $\phi'(z) = \mathbf{1}_{z>0}$, so
$D^2_i = \mathbb{E}[\mathbf{1}_{z_i>0}]$ is the mean activation probability
in channel $i$ across the test set.
For the linear model: $D^2 = I$ exactly (no hook required).
All tensors are moved to CPU float32 before commutator computation.

\paragraph{Dimension sanity check.}
Before computing the commutator we assert
$D^2_\mathrm{diag}.\mathrm{shape}[0] = U_\mathrm{feat}.\mathrm{shape}[0]$
to guard against hooks firing on the wrong layer.
We also check orthonormality of $U_\mathrm{feat}$:
$\|U_\mathrm{feat}^\top U_\mathrm{feat} - I_r\|_F < 10^{-4}$.

\subsubsection{Commutator Computation}

Using the identity
$[D^2, P_V]_{ij} = (D^2_{ii} - D^2_{jj})(P_V)_{ij}$
for diagonal $D^2$, we compute

\begin{equation}
\|[D^2, P_V]\|_F^2
= \sum_{i,j} (D^2_{ii} - D^2_{jj})^2 (P_V)_{ij}^2,
\label{eq:appE_commutator_expansion}
\end{equation}

via a broadcasting outer-difference, avoiding explicit construction
of the $d \times d$ commutator matrix.
The relative commutator is normalized as

\begin{equation}
C_\mathrm{rel} = \frac{\|[D^2,P_V]\|_F}{\|A\|_F \cdot \sqrt{r} + \varepsilon},
\quad A = D^2 - cI,
\label{eq:appE_Crel}
\end{equation}

where $c = \mathrm{tr}(D^2)/d$ and $\|P_V\|_F = \sqrt{r}$. We set $\varepsilon = 10^{-8}$ to avoid
division by zero.

For architectures with non-degenerate spectrum, the Level-3 alignment
score is
\[
A_3 = \frac{1}{r}\sum_{i=1}^r
|\langle (U_\mathrm{feat})_i,\,(\tilde{U}_\mathrm{class})_i\rangle|,
\]
where $\tilde{U}_\mathrm{class} = U_\mathrm{class} Q$ and
$U_\mathrm{feat}^\top U_\mathrm{class} = P\Sigma Q^\top$ is the SVD
used in the basis-matching procedure.
Table~\ref{tab:app_vit_l23} reports the excess $A_3-\mu_0$ over
the Haar-random baseline mean.
When the feature spectrum has a relative gap below $0.1$, Level-3 is
reported as ill-defined.

\subsubsection{Full Results}

\begin{table}[H]
\centering
\caption{Full mechanistic results across all seeds.
$c$: conformality scalar $\mathrm{tr}(D^2)/d$.
Rel.\ gap: $\|D^2-cI\|_F / (c\sqrt{d})$.
$C_\mathrm{rel}$: relative commutator.}
\label{tab:app_commutator_full}
\vspace{4pt}
\resizebox{\textwidth}{!}{\begin{tabular}{lccccl}
\toprule
Model & $c$ & $\|D^2-cI\|_F$ & Rel.\ gap
      & $\|[D^2,P_V]\|_F$ & $C_\mathrm{rel}$ \\
\midrule
MLP-linear-50 (5 seeds)
  & $1.000 \pm 0.000$
  & $0.000 \pm 0.000$
  & $0.000 \pm 0.000$
  & $0.000 \pm 0.000$
  & $0.000 \pm 0.000$ \\
MLP-ReLU-10 Xavier (5 seeds)
  & $0.648 \pm 0.071$
  & \phantom{0}---
  & \phantom{0}---
  & $1.219 \pm 0.334$
  & --- \\
MLP-ReLU-50 orth (15 seeds)
  & $0.425 \pm 0.038$
  & \phantom{0}---
  & \phantom{0}---
  & $1.461 \pm 0.075$
  & --- \\
\bottomrule
\end{tabular}}
\end{table}

\noindent
\textbf{Note on MLP-ReLU-50 Xavier (excluded).}
We also evaluated the 50-layer ReLU MLP with Xavier initialization
(15 seeds). Two seeds produced $D^2_\mathrm{mean} \approx 0$
and commutator $\approx 0$ due to dying neurons; the remaining seeds
had $D^2_\mathrm{mean} \in [0.008, 0.050]$, far below the orthogonal-init
model. These results reflect a training instability rather than the
commutator obstruction and are excluded from the main comparison.
This is consistent with the known failure of Xavier initialization
for very deep ReLU networks~\cite{he2015delving}.

\paragraph{Per-seed detail: MLP-ReLU-50 (orth).}

\begin{table}[H]
\centering
\small
\begin{tabular}{rcc}
\toprule
Seed & $\|[D^2,P_V]\|_F$ & $D^2_\mathrm{mean}$ \\
\midrule
0  & 1.576 & 0.372 \\
1  & 1.489 & 0.447 \\
2  & 1.385 & 0.445 \\
3  & 1.408 & 0.408 \\
4  & 1.468 & 0.427 \\
5  & 1.460 & 0.450 \\
6  & 1.431 & 0.401 \\
7  & 1.455 & 0.390 \\
8  & 1.467 & 0.444 \\
9  & 1.475 & 0.419 \\
10 & 1.573 & 0.464 \\
11 & 1.298 & 0.360 \\
12 & 1.382 & 0.496 \\
13 & 1.524 & 0.445 \\
14 & 1.530 & 0.399 \\
\midrule
Mean $\pm$ std & $1.461 \pm 0.075$ & $0.425 \pm 0.038$ \\
\bottomrule
\end{tabular}
\caption{Per-seed results for MLP-ReLU-50 (orthogonal init, 15 seeds).}
\label{tab:app_relu50_seeds}
\end{table}

\paragraph{Evaluation metrics.}
The \emph{Level-2 z-score} measures how many standard deviations 
the observed Level-2 alignment exceeds the null distribution under 
random subspace assignment. Specifically, $z = (\hat\mu - \mu_0)/\sigma_0$, 
where $\hat\mu$ is the observed alignment, and $\mu_0$, $\sigma_0$ 
are the mean and standard deviation under the Haar-random baseline.
A run is classified as \texttt{gc\_acceptable} if the geometric 
consistency gap satisfies $\|D^2 - cI\|_F / (c\sqrt{d}) < \tau$, 
with threshold $\tau = 0.5$.

\subsubsection{Level-2 Alignment and Geometric Consistency: CNN Architectures}

\begin{table}[H]
\centering
\caption{Level-2 z-scores for CNN architectures (mean $\pm$ std across seeds).
Level-3 is ill-defined (spectral collapse) for all CNN models.
All seeds satisfy \texttt{gc\_acceptable} for ResNet-18 and ResNet-50.}
\label{tab:app_cnn_l2}
\vspace{4pt}
\resizebox{\textwidth}{!}{\begin{tabular}{llccc}
\toprule
Model & Activation & Level-2 $z$-score & Level-3 status & gc\_acceptable \\
\midrule
ResNet-18 & ReLU      & $335.7 \pm 0.3$ & Degenerate & $20/20$ \\
          & GELU      & $303.5 \pm 0.6$ & Degenerate & $20/20$ \\
          & LeakyReLU & $323.1 \pm 0.4$ & Degenerate & $20/20$ \\
          & Mish      & $301.6 \pm 0.5$ & Degenerate & $20/20$ \\
          & SiLU      & $316.7 \pm 0.7$ & Degenerate & $20/20$ \\
          & Softplus  & $320.8 \pm 0.5$ & Degenerate & $20/20$ \\
\midrule
ResNet-50 & ReLU      & $1293.8 \pm 14.5$ & Degenerate & $10/10$ \\
          & GELU      & $1341.3 \pm 6.0$  & Degenerate & $10/10$ \\
          & LeakyReLU & $1194.5 \pm 5.8$  & Degenerate & $10/10$ \\
          & Mish      & $1247.8 \pm 3.2$  & Degenerate & $8/8$   \\
          & SiLU      & $1224.7 \pm 4.4$  & Degenerate & $10/10$ \\
          & Softplus  & $1299.5 \pm 8.6$  & Degenerate & $9/9$   \\
\midrule
ConvNeXt-Tiny~\cite{liu2022convnet} & ReLU  & $429.8 \pm 10.5$ & Degenerate & $24/24$ \\
              & GELU  & $447.6 \pm 9.7$  & Degenerate & $24/24$ \\
              & SiLU  & $446.4 \pm 7.9$  & Degenerate & $24/24$ \\
\midrule
MLP-Mixer-B16~\cite{tolstikhin2021mlpmixer} & ReLU & $452.5 \pm 53.6$ & Degenerate & $9/24$  \\
              & GELU & $450.3 \pm 26.5$ & Degenerate & $12/24$ \\
              & SiLU & $395.0 \pm 36.3$ & Degenerate & $16/24$ \\
\bottomrule
\end{tabular}}
\end{table}
\begin{figure}[H]
\centering
\includegraphics[width=\textwidth]{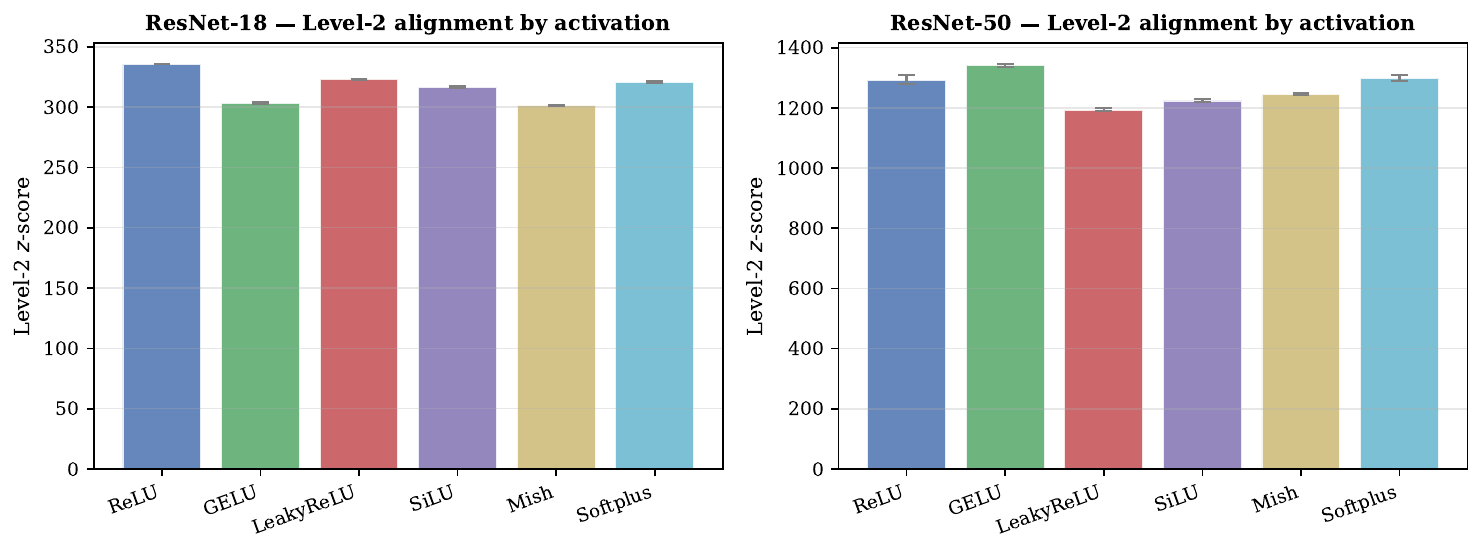}
\caption{Level-2 alignment z-scores for ResNet-18 and ResNet-50
across activation functions (mean $\pm$ std). All activations show
high Level-2 z-scores, consistent with Theorem~\ref{thm:l2_attraction}.}
\label{fig:app_resnet_l2}
\end{figure}
\noindent\textbf{Note on MLP-Mixer geometric consistency.}
MLP-Mixer shows \texttt{gc\_acceptable} in approximately 50\% of seeds,
indicating that Assumption~\ref{ass:geometric_consistency} is only
partially satisfied. This is consistent with the token-mixing architecture
of MLP-Mixer, where spatial mixing interacts with channel-wise activations
in a non-standard way. MLP-Mixer results are therefore reported here for
completeness but are not used to support the main theoretical claims.
\begin{figure}[H]
\centering
\includegraphics[width=\textwidth]{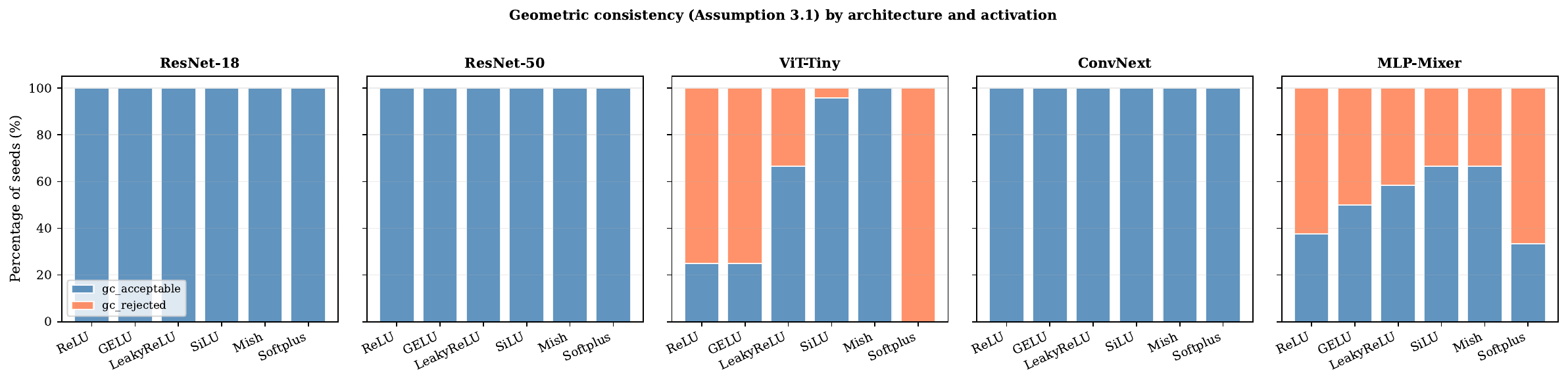}
\caption{Geometric consistency (\texttt{gc\_acceptable} rate) by
architecture and activation. ResNet and ConvNext satisfy
Assumption~\ref{ass:geometric_consistency} uniformly; ViT-Tiny and
MLP-Mixer show activation-dependent violations.}
\label{fig:app_gc}
\end{figure}

\subsubsection{ViT-Tiny: Level-2 and Level-3 Alignment}

\begin{table}[H]
\centering
\caption{ViT-Tiny Level-2 z-score and Level-3 mean alignment across
activations (mean $\pm$ std, 8 seeds each). Level-3 is well-defined
but near zero, consistent with the commutator obstruction.
gc\_acceptable rate reflects whether Assumption~\ref{ass:geometric_consistency}
holds at the penultimate activation layer.}
\label{tab:app_vit_l23}
\vspace{4pt}
\begin{tabular}{lcccl}
\toprule
Activation & Level-2 $z$ & Level-3 alignment & gc\_acceptable & Note \\
\midrule
ReLU      & $43.6 \pm 3.9$   & $-0.017 \pm 0.085$ & $6/24$  & gc mostly rejected \\
GELU      & $59.5 \pm 12.6$  & $-0.059 \pm 0.090$ & $6/24$  & gc mostly rejected \\
SiLU      & $80.7 \pm 5.0$   & $-0.006 \pm 0.085$ & $23/24$ & gc acceptable \\
LeakyReLU & $82.7 \pm 7.4$   & $+0.011 \pm 0.074$ & $16/24$ & gc partial \\
Mish      & $83.8 \pm 4.2$   & $+0.068 \pm 0.067$ & $24/24$ & gc acceptable \\
Softplus  & $80.2 \pm 6.1$   & $+0.050 \pm 0.114$ & $0/24$  & gc all rejected \\
\bottomrule
\end{tabular}
\end{table}
\begin{figure}[H]
\centering
\includegraphics[width=\textwidth]{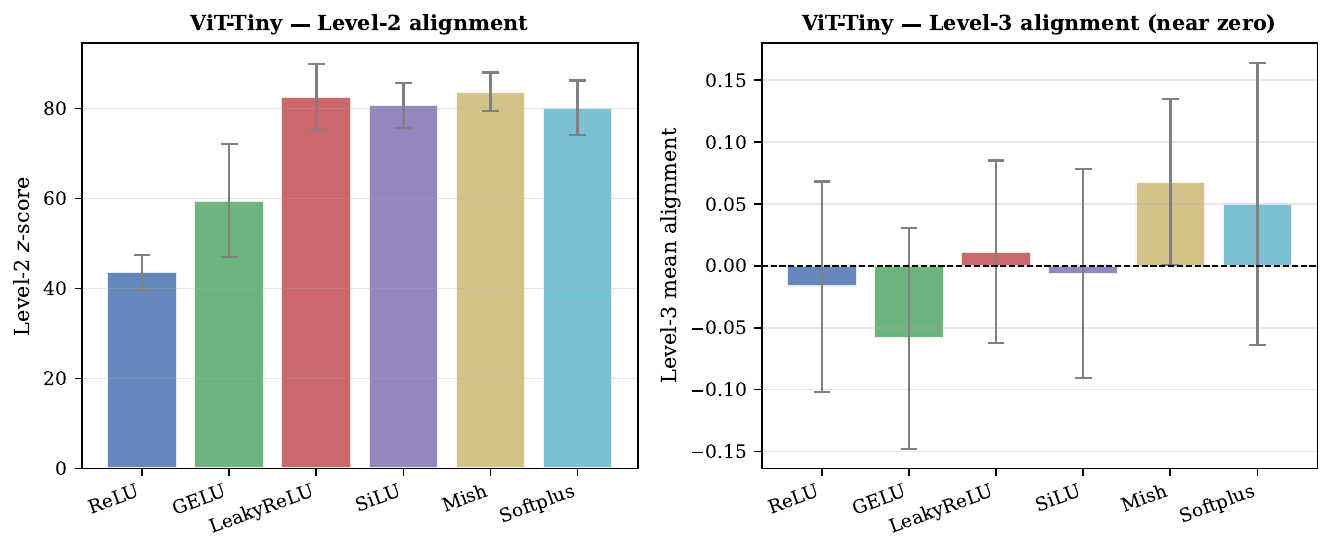}
\caption{ViT-Tiny Level-2 z-score (left) and Level-3 mean alignment
(right) across activations. Level-3 alignment remains near zero
regardless of activation, consistent with the commutator obstruction.}
\label{fig:app_vit}
\end{figure}
\noindent Level-3 alignment is well-defined for ViT-Tiny (non-degenerate spectrum)
but the mean alignment score is near zero regardless of activation.
The activation-dependence of the gc criterion reflects that ViT's
self-attention layers introduce a different activation metric structure
than standard feedforward networks; activations with near-isotropic
Jacobians (Mish, SiLU) satisfy Assumption~\ref{ass:geometric_consistency}
more reliably.

\subsection{Part B: Transformer Head Subspace Diagnostic}
\label{app:transformer_diagnostic}

\subsubsection{Models}

\begin{description}[leftmargin=1.2em,itemsep=0.3em]
\item[\textbf{GPT-2 small (pretrained).}]
  $d=768$, $h=12$, $d_k=64$, $L=12$.
  Loaded from HuggingFace Transformers~\cite{wolf2020transformers} (\texttt{gpt2}).
  Analytic baseline: $d_k/d = 0.0833$.
\item[\textbf{GPT-2 small (random init).}]
  Identical architecture, randomly initialized weights.
  Used as a sanity check to verify the null baseline.
\item[\textbf{Llama-3.2-1B (pretrained).}]
  $d=2048$, $h=32$, $d_k=64$, $L=16$.
  Loaded from HuggingFace (\texttt{meta-llama/Llama-3.2-1B})
  in \texttt{bfloat16}; per-layer conversion to \texttt{float32}
  for QR decomposition.
  Analytic baseline: $d_k/d = 0.0313$.
\item[\textbf{DistilGPT-2 (pretrained).}]
  $d=768$, $h=12$, $d_k=64$, $L=6$.
  Same architecture as GPT-2 small, fewer layers.
  Mean $\bar R = 0.185$ ($6/6$ layers above $2\sigma$).
  Included as supplementary evidence.
\item[\textbf{GPT-2 medium (pretrained).}]
  $d=1024$, $h=16$, $d_k=64$, $L=24$.
  Mean $\bar R = 0.147$ ($24/24$ layers above $2\sigma$ vs same-architecture
  random baseline).
  Note: $\bar R$ is lower than GPT-2 small, possibly due to greater head
  specialization at larger scale; we do not claim a size-monotone relationship.
\end{description}

\subsubsection{Metric and Extraction}

For each layer $\ell$, we extract the query projection matrix
$W_Q^{(\ell,i)} \in \mathbb{R}^{d \times d_k}$ for each head $i$
and compute its orthonormal column basis $U_i^{(\ell)}$ via
QR decomposition (\texttt{torch.linalg.qr}, reduced mode).
QR is preferred over SVD because we require only the column space,
and QR is substantially faster on large $d$ (e.g., Llama).

\paragraph{GPT-2 weight extraction.}
GPT-2 uses a combined \texttt{Conv1D} projection with weight stored
as $(d, 3d)$ rather than the transposed $(3d, d)$ convention of
\texttt{nn.Linear}.
Query weights occupy the first $d$ columns:
$W_Q = \texttt{c\_attn.weight}[:, :d] \in \mathbb{R}^{d \times d}$.

\paragraph{Llama weight extraction.}
Llama uses separate \texttt{q\_proj}, \texttt{k\_proj}, \texttt{v\_proj}.
$W_Q = \texttt{q\_proj.weight}^\top \in \mathbb{R}^{d \times h d_k}$.
We analyze query projections only; for grouped-query architectures
this uses \texttt{num\_attention\_heads} (query heads),
not \texttt{num\_key\_value\_heads}.

\paragraph{$d_k$ verification.}
We read $d_k = q_\mathrm{out} / h$ from the actual weight shape
$\texttt{q\_proj.weight.shape} = (h \cdot d_k,\, d)$
rather than assuming $d_k = d/h$, to handle non-standard configurations.

\subsubsection{Random Baseline}

The Haar-random baseline is computed empirically by drawing
$n_\mathrm{trials}$ independent sets of $h$ random orthonormal
$d_k$-frames in $\mathbb{R}^d$ (via QR of random Gaussian matrices),
computing $R^{(\ell)}$ for each, and reporting
mean $\pm$ std across trials and layers.
This agrees with the analytic expectation
$\mathbb{E}[R^{(\ell)}] = d_k/d$ to four decimal places
($n_\mathrm{trials} = 100$ for GPT-2, $50$ for Llama).

\paragraph{Unordered-pair equivalence.}
The formula $\frac{1}{h(h-1)}\sum_{i \neq j}$ averages over
ordered pairs. In code we average over unordered pairs $i < j$
(dividing by $h(h-1)/2$); the two formulas are equivalent.

\subsubsection{Full Layerwise Results}

\begin{table}[H]
\centering
\small
\caption{Layerwise $R^{(\ell)}$ for GPT-2 small and Llama-3.2-1B.
Random baseline: $0.0834 \pm 0.0002$ (GPT-2) and $0.0313 \pm 0.0000$ (Llama).}
\label{tab:app_transformer_full}
\begin{tabular}{rcc}
\toprule
Layer $\ell$ & GPT-2 small & Llama-3.2-1B \\
\midrule
0  & 0.1492 & 0.1045 \\
1  & 0.1740 & 0.0740 \\
2  & 0.1732 & 0.0694 \\
3  & 0.1780 & 0.0825 \\
4  & 0.1767 & 0.0809 \\
5  & 0.1778 & 0.0970 \\
6  & 0.1851 & 0.1003 \\
7  & 0.1882 & 0.1117 \\
8  & 0.1925 & 0.1263 \\
9  & 0.1938 & 0.1018 \\
10 & 0.2350 & 0.0958 \\
11 & 0.2303 & 0.0907 \\
12 & ---    & 0.0801 \\
13 & ---    & 0.0867 \\
14 & ---    & 0.1732 \\
15 & ---    & 0.1304 \\
\midrule
Mean & $0.1878$ & $0.1003$ \\
Layers $> 2\sigma$ & $12/12$ & $16/16$ \\
\bottomrule
\end{tabular}
\end{table}

\begin{table}[H]
\centering
\small
\caption{Multi-model summary (same-architecture random baseline
for GPT-family; Haar baseline for Llama).}
\label{tab:app_multimodel}
\begin{tabular}{lcccc}
\toprule
Model & $L$ & $d_k/d$ & $\bar R$ & Layers $> 2\sigma$ \\
\midrule
GPT-2 small (random init) & 12 & 0.083 & 0.083 & $0/12$ \\
DistilGPT-2               &  6 & 0.083 & 0.185 & $6/6$  \\
GPT-2 small               & 12 & 0.083 & 0.188 & $12/12$ \\
GPT-2 medium              & 24 & 0.063 & 0.147 & $24/24$ \\
Llama-3.2-1B              & 16 & 0.031 & 0.100 & $16/16$ \\
\bottomrule
\end{tabular}
\end{table}

\paragraph{GPT-2 medium note.}
GPT-2 medium has a lower mean $\bar R = 0.147$ than GPT-2 small
($\bar R = 0.188$), despite being larger.
This may reflect greater head specialization at scale,
but we do not claim a size-monotone relationship and do not
pursue this further here.

\subsubsection{Interpretation and Limitations}

The consistent excess of $R^{(\ell)}$ over the random baseline
across architectures (GPT-2 family and Llama) and training domains
(language modeling) supports the geometric-consistency principle
as a broadly applicable diagnostic.

We caution against over-interpreting these results:
\begin{enumerate}[leftmargin=1.5em,itemsep=0.2em]
\item These are \emph{terminal-state} observations only.
  We do not have training checkpoints and cannot verify
  whether $R^{(\ell)}$ increases monotonically during training.
\item The Haar-random baseline controls for subspace dimension
  but not for the specific structure of transformer training objectives.
  A random model trained on a simpler task might also show excess overlap.
\item $R^{(\ell)}$ aggregates pairwise overlaps and cannot distinguish
  a configuration where all heads partially overlap from one where
  a subset of heads are near-identical.
\end{enumerate}

\end{document}